\documentclass[preprint,12pt]{elsarticle}

\usepackage{amssymb}

\usepackage{subfigure}
\usepackage{amsfonts}
\usepackage{amsmath}
\usepackage{amsthm}
\usepackage{amssymb}
\usepackage{color}
\usepackage{booktabs}
\usepackage{array}
\usepackage{colortbl}
\usepackage{multirow}
\usepackage{color}
\usepackage{bm}
\usepackage[ruled, vlined]{algorithm2e}

\usepackage[colorlinks=true, citecolor=blue]{hyperref}  
\newtheorem{theorem}{Theorem}

\newtheorem{lemma}{Lemma}
\newtheorem{corollary}{Corollary}

\def\sph2{\mathbb{S}^{2}}

\def\b{\mathbf}

\def\RR{\mathbb{R}}

\def\P{\mathbf{P}}

\begin{document}

\begin{frontmatter}



\title{Bridging Smoothness and Approximation: Theoretical Insights into Over-Smoothing in Graph Neural  Networks}

 \author[1]{Guangrui Yang}
\ead{yanggrui@mail2.sysu.edu.cn}
\author[2]{Jianfei Li}
\ead{jianfeili2-c@my.cityu.edu.hk}
\author[3,4]{Ming Li}
\ead{mingli@zjnu.edu.cn}
\author[2]{Han Feng\corref{cor1}}
\ead{hanfeng@cityu.edu.hk}
\author[5]{Dingxuan Zhou}
\ead{dingxuan.zhou@sydney.edu.au}
\cortext[cor1]{Corresponding author}
\address[1]{Department of Mathematics, College of Mathematics and Informatics, South China Agricultural University, Guangzhou, China.}
\address[2]{Department of Mathematics, City University of Hong Kong, Hong Kong, China}
\address[3]{Zhejiang Key Laboratory of Intelligent Education Technology  and Application, Zhejiang Normal University, Jinhua 321004, China}
\address[4]{Zhejiang Institute of Optoelectronics, Jinhua 321004, China}
\address[5]{School of Mathematics and Statistics, The University of Sydney, Sydney, New South Wales 2006, Australia}

\begin{abstract}
In this paper, we explore the approximation theory of functions defined on graphs. Our study builds upon the approximation results derived from the $K$-functional. We establish a theoretical framework to assess the lower bounds of approximation for target functions using Graph Convolutional Networks (GCNs) and examine the over-smoothing phenomenon commonly observed in these networks. Initially, we introduce the concept of a $K$-functional on graphs, establishing its equivalence to the modulus of smoothness. We then analyze a typical type of GCN to demonstrate how the high-frequency energy of the output decays, an indicator of over-smoothing. This analysis provides theoretical insights into the nature of over-smoothing within GCNs. Furthermore, we establish a lower bound for the approximation of target functions by GCNs, which is governed by the modulus of smoothness of these functions. This finding offers a new perspective on the approximation capabilities of GCNs. In our numerical experiments, we analyze several widely applied GCNs and observe the phenomenon of energy decay. These observations corroborate our theoretical results on exponential decay order.
\end{abstract}




\end{frontmatter}


\section{Introduction}
Graph Convolutional Networks (GCNs) have emerged as a powerful category of deep learning models that extend the prowess of traditional neural network architectures to data that are structured as graphs \cite{gcn,huang2022graph}. Such data encapsulate relationships and interdependencies between entities, which are crucial in domains like social networks, biological networks, recommendation systems, and more. The ability of GCNs to leverage the connectivity and features of graph data has led to significant advancements and state-of-the-art results in various tasks \cite{bacciu2020gentle,Survey_SunMS, Survey_ZhangCQ, Survey_ZhuWW,li2024guest,Zhou2024}.

However, unlike other deep learning models, graph neural networks did not, until recently, benefit significantly from increasing depths. As GCNs deepen, a common issue they encounter is known as over-smoothing \cite{rusch2023survey,2021NeurIPS_Dirichlet,2022pmlr,2020AAAI,2020Towards,2021NEURIPS_OnProvable}, where the node representations become indistinguishable and converge to a subspace that lacks the discriminative power necessary for downstream tasks. This phenomenon fundamentally limits the ability of GCNs to capture higher-order dependencies with graphs, and has attracted intense research attention \cite{cai2020note,keriven2022not,oono2019graph,Huang2024}. 
Li et al. \cite{Li2018} demonstrates that the graph convolution in the GCN model is essentially a specific type of Laplacian smoothing. This mechanism underpins the effectiveness of GCNs; however, it also introduces the risk of over-smoothing when numerous convolutional layers are employed. Chen et al. \cite{cai2020note} investigated the problem of over-smoothing using Dirichlet energy, noting that for deep GCN models, the Dirichlet energy of the embeddings tends to approach zero, which leads to a reduction in their ability to distinguish between different features. Wu et al. \cite{wu2024demystifying} prove that the graph attention mechanism cannot prevent oversmoothing and loses expressive power exponentially. For more detailed information on the formal definition of over-smoothing, associated metrics, and existing strategies for reducing it, we refer the readers to the recent comprehensive survey by Rusch et al. \cite{rusch2023survey}.


The work presented in this paper addresses the theoretical underpinnings of the over-smoothing phenomenon through the lens of approximation theory on graph signals. The approximation theory for signals on graphs has arisen in recent literature. Building on Schrödinger's semigroup of operators, I. Z. Pesenson et al. \cite{2010Pesenson} established a sparse approximation theory on graphs. In \cite{2012Zhu}, the authors defined a metric to measure the smoothness of a graph signal, and they found that both the eigenvalues of the graph Laplacian and graph total variation affect the upper bound of the $M$-term linear Fourier approximation error. Furthermore, L. Avena et al. \cite{2018Avena,2020Avena} explored the multiscale analysis of functions on graphs and proved a Jackson-like inequality. In \cite{2009Pesenson,2010Pesenson_Remset,2010Pesenson_Eigenmaps},  I. Z. Pesenson et al. studied the approximation of functions on the graph using graph splines defined by a variational problem. In a recent work \cite{2021_appHuang}, the authors proposed the definition of the modulus of smoothness on the graph by the graph translation operators, and then proved the Jackson's and Bernstein's inequalities for functions defined on graphs. These works have contributed greatly to the establishment of approximation theory on graphs. 

In this paper, we introduce the definition of $K$-functional on graphs and then establish a strong equivalence relation between  moduli of smoothness and $K$-functionals for graph signals. These concepts, deeply rooted in approximation theory, are critical in understanding behaviors of functions on discrete structures like graphs, particularly in the context of their smoothness and how well they can be approximated by bandlimited functions.

By applying this theoretical framework, we establish a novel lower estimation of the approximation property of graph neural networks. This estimation is not just a measure of the ability of GNNs to approximate graph signals but also a diagnostic tool to assess the extent to which over-smoothing may be affecting a given GCN architecture.


The main contributions of this paper are listed in the following.

\begin{itemize}
  \item Theoretical Framework: We established a robust theoretical framework that equates $K$-functionals with moduli of smoothness on graphs. This framework is instrumental in clarifying how GCNs process and transform data.
\item Decay of High-Frequency Energy: Our analysis revealed that the decay of high-frequency energy in GCNs is a critical factor contributing to over-smoothing. This insight not only expands our understanding of the dynamics within GCNs but also suggests potential modifications to control or mitigate over-smoothing effects.
  \item Lower Bounds of Approximation: By establishing lower bounds for the approximation of target functions, governed by moduli of smoothness, we provided quantifiable metrics that can guide the development and tuning of GCNs to achieve desired approximation accuracies.
\end{itemize}

The rest of this paper is structured as follows: Section \ref{sec:Pre} provides background on graph signal processing and introduces moduli of smoothness and $K$-functionals in the context of graph signals. Section \ref{sec:app_theory} details our main approximation theory results for $K$-functionals on the graph, laying out a strong equivalence relationship between 
$K$-functionals and moduli of smoothness. In Section \ref{sec:lower_bound}, we analyze a typical type of GCNs to demonstrate how the high-frequency energy of the output decays indicating over-smoothing phenomenon,  and then based on the approximation results of $K$-functional, we prove a lower bound of approximating the target functions by GCNs. In Section \ref{sec:Exp}, we present numerical experiments to validate our results. Finally, Section \ref{sec:conclusion} gives a brief conclusion of this paper and some discussions on strategies to mitigate the over-smoothing problem and enhance the approximation capability of GCNs.

\section{ Notations and Preliminaries} \label{sec:Pre}
This section recalls some necessary notations and concepts for subsequent sections.

\subsection{Graph Signal Processing}
Let \(\mathcal{G} = \{\mathcal{V}, \mathbf{A}\}\) be an undirected, weighted, and connected graph, where \(\mathcal{V} = \{\mathbf{v}_1, \mathbf{v}_2, \ldots, \mathbf{v}_N\}\) denotes the vertices of the graph, and \(\mathbf{A} = (A_{ij}) \in \mathbb{R}^{N \times N}\) is the adjacency matrix. Here, \(A_{ij} > 0\) if there is an edge connecting vertex \(\mathbf{v}_i\) and \(\mathbf{v}_j\), and \(A_{ij} = 0\) otherwise. The graph Laplacian of \(\mathcal{G}\) is defined as \(\mathbf{L} = \mathbf{D} - \mathbf{A}\), where \(\mathbf{D} = \operatorname{diag}(d_1, d_2, \ldots, d_N)\) is the diagonal degree matrix with \(d_i = \sum_{j=1}^N A_{ij}\) for \(i = 1, 2, \ldots, N\).

A graph signal defined on \(\mathcal{G}\) is a function \(f : \mathcal{V} \rightarrow \mathbb{C}\), where \(f(\mathbf{v}_i)\) denotes the signal value at node \(\mathbf{v}_i\). In this paper, we only consider graphs of finite size \(N\), so the signal \(f\) is also equivalent to a vector \(f \in \mathbb{C}^N\). For convenience, we use \(f(i)\) to represent \(f(\mathbf{v}_i)\) in the rest of this paper. Since \(\mathbf{L}\) is real and symmetric, it can be orthogonally diagonalized as \(\mathbf{L} = \mathbf{U} \Lambda \mathbf{U}^*\), where \(\mathbf{U} = (\mathbf{u}_1, \mathbf{u}_2, \ldots, \mathbf{u}_N) \in \mathbb{C}^{N \times N}\), \(\Lambda = \operatorname{diag}(\lambda_1, \lambda_2, \ldots, \lambda_N) \in \mathbb{R}^{N \times N}\) with \(0 = \lambda_1 < \lambda_2 \leq \cdots \leq \lambda_N = \lambda_{\max}\), and \(\mathbf{U}^*\) denotes the conjugate transpose of \(\mathbf{U}\).

The matrix \(\mathbf{U}\) (called the graph Fourier basis) is often used to define the Graph Fourier Transform (GFT) \cite{2013TheEmerging,2021Yanggrui}. Specifically, for any \(f \in \mathbb{C}^N\), its GFT is defined as \(\hat{f} = \mathbf{U}^* f\), and the inverse graph Fourier transform is \(f = \mathbf{U} \hat{f}\). In graph signal processing, the pair \(\{\lambda_j, \mathbf{u}_j\}\) is often referred to as the frequency and frequency component of graph \(\mathcal{G}\). Based on the GFT, the function \(f = \mathbf{U} \hat{f} = \sum_{j=1}^N \hat{f}(j) \mathbf{u}_j\) is decomposed into the sum of \(N\) components \(\{\hat{f}(j) \mathbf{u}_j\}_{j=1}^N\), where \(\hat{f}(j)\) is called the spectrum of \(f\) corresponding to the frequency \(\lambda_j\).

We denote by $\operatorname{PW}_{\lambda_{n}}(\mathcal{G})$ the signal space consisting of functions $f$ satisfying $\hat{f}(j)=0$ for all $j>n$, that is,

$$
\operatorname{PW}_{\lambda_{n}}(\mathcal{G})=\operatorname{span}\left(\mathbf{u}_{1}, \mathbf{u}_{2}, \ldots, \mathbf{u}_{n}\right).
$$
The space $\textrm{PW}_{\lambda_{n}}(\mathcal{G})$ is the well-known Paley-Wiener space, and its orthogonal projection operator is denoted by $\mathbf{P}_{n}$, i.e., $\mathbf{P}_{n}=\mathbf{U}_{n} \mathbf{U}_{n}^{*}$, where $\mathbf{U}_{n}=$ $\left(\mathbf{u}_{1}, \mathbf{u}_{2}, \ldots, \mathbf{u}_{n}\right) \in \mathbb{C}^{N \times n}.$ 

For $n=1,2, \ldots, N$, the best approximation error (with 2 -norm) of $f$ by the functions in $P W_{\lambda_{n}}(\mathcal{G})$ is

$$
E_{n}(f):=\min _{g \in P W_{\lambda_{n}}(\mathcal{G})}\|f-g\|_{2}
$$

\subsection{Moduli of Smoothness and $K$-functionals on Graphs}

We now define the moduli of smoothness on graphs. For $s \in \mathbb{R}$,
\begin{equation}\label{equ:translation}
\mathbf{T}_{s} := e^{i s \sqrt{\mathbf{L}}} = \mathbf{U} \operatorname{diag}\left(e^{is \sqrt{\lambda_{1}}}, e^{i s \sqrt{\lambda_{2}}}, \ldots, e^{i s \sqrt{\lambda_{N}}}\right) \mathbf{U}^{*}
\end{equation}
defines a family of graph translation operators. In \cite{2021_appHuang}, using the definition of graph translation operator, the modulus of smoothness (with 2-norm) of order $r\in\mathbb{N}$ is defined as
\[
\omega_{r}(f, t) := \sup _{|s| \leq t} \left\|\triangle_{s}^{r} f\right\|_{2},~~t>0,
\]
where $\triangle_{s}^{r}: \mathbb{C}^{N} \rightarrow \mathbb{C}^{N}$ is the difference operator of order $r$, defined by
\[
\triangle_{s}^{r} f := \left(\mathbf{T}_{s} - \mathbf{I}\right)^{r} f, \quad f \in \mathbb{C}^{N},
\]
with $\mathbf{I}$ denoting the identity matrix. Note that in \cite{2021_appHuang}, the modulus of smoothness is defined using a general $p$-norm, but in this paper, we focus only on the modulus of smoothness with the 2-norm. According to the results in \cite{2021_appHuang}, $\omega_{r}(f, t)$ has the following four properties:

1. $\omega_{r}(f, \lambda t) \leq (1+\lambda)^{r} \omega_{r}(f, t)$;

2. $\omega_{r}(f_{1} + f_{2}, t) \leq \omega_{r}(f_{1}, t) + \omega_{r}(f_{2}, t)$;

3. $\omega_{r}(f, t) \leq 2^{j} \omega_{r-j}(f, t)$ for $1 \leq j \leq r$, where $\omega_{0}(f,t)=\|f\|_{2}$;

4. $\omega_{r}(f, t) \leq t^{j} \omega_{r-j}\left(\mathbf{L}^{j / 2} f, t\right)$ for $1 \leq j \leq r$.

Building upon the modulus of smoothness, the authors in \cite{2021_appHuang} proved that $E_{n}(f)$ can be controlled by $\omega_{r}(f, t)$, as shown in Lemma~\ref{lem:best_app}.

\begin{lemma}\label{lem:best_app}
     For $n = 2, 3, \ldots, N$,
\[
E_{n}(f) \leq C_{r}^{\prime} \omega_{r}\left(f, \lambda_{n}^{-1 / 2}\right),
\]
where $C_{r}^{\prime} = \frac{4}{\pi}(r+3)^{r}(r+1)$.
\end{lemma}

In classical approximation theory, there exists an important equivalent relationship between a modulus of smoothness and a $K$-functional. Now, we aim to extend these results to graphs. We first define the $K$-functional of order $r$ on the graph as follows:
\begin{equation}
K_{r}(f, t) := \min _{g \in \mathbb{C}^{N}} \|f-g\|_{2} + \left(\frac{t}{2}\right)^{r} \left\|\mathbf{L}^{r / 2} g\right\|_{2},~~t>0. 
\end{equation}
The second term above measures, to some extent, the smoothness of $g$ on the graph \cite{2013TheEmerging}. Therefore, there shall exist a special relationship between the modulus of smoothness $\omega_{r}(f, t)$ and $K$-functional $K_{r}(f, t)$ on the graph. Specifically, in the next section, we will prove an equivalence relationship between $\omega_{r}(f, t)$ and $K_{r}(f, t)$.
\section{Approximation theory for $K$-functional on the graph} \label{sec:app_theory}



In this section, we will establish a strong equivalence relation between the $K$-functional and the modulus of smoothness of order $r$ for graph signals as demonstrated in Theorem \ref{thm:main}. Prior to this, two lemmas are introduced, which play a crucial role in supporting the proof of Theorem \ref{thm:main}.

\begin{lemma}\label{lem:1}
Let $n=2,3,\dots,N$. For any $f\in\mathbb{C}^{N}$,
$$(2/\pi)^{r}\lambda_{n}^{-r/2}\|\mathbf{L}^{r/2}\mathbf{P}_{n}f\|_{2}\leq
   \|(\mathbf{T}_{\lambda_{n}^{-1/2}}-\mathbf{I})^{r}\mathbf{P}_{n}f\|_{2}\leq
  \lambda_{n}^{-r/2}\|\mathbf{L}^{r/2}\mathbf{P}_{n}f\|_{2}.$$
 Furthermore, 
  \begin{equation}\label{equ:omega_vs_L}
 \lambda_{n}^{-r/2}\|\mathbf{L}^{r/2}\mathbf{P}_{n}f\|_{2}\leq(\pi/2)^{r}\omega_{r}(f,\lambda_{n}^{-1/2}).
  \end{equation}
\end{lemma}

\begin{proof}
Observe from $\mathbf{U}^{*}\mathbf{U}_{n}=\left[
   \begin{array}{c}
     \mathbf{I}_{n} \\
     \mathbf{0} \\
   \end{array}
 \right]$ and $$(\mathbf{T}_{\lambda_{n}^{-1/2}}-\mathbf{I})^{r}=\mathbf{U}\textrm{diag}\Big((e^{i\sqrt{\frac{\lambda_{1}}{\lambda_{n}}}}-1)^r,\dots,(e^{i\sqrt{\frac{\lambda_{n}}{\lambda_{n}}}}-1)^r,\dots,(e^{i\sqrt{\frac{\lambda_{N}}{\lambda_{n}}}}-1)^r\Big)\mathbf{U}^{*}$$
 with the $n\times n$ identity matrix $\mathbf{I}_{n}$, we have
 \begin{align*}
 (\mathbf{T}_{\lambda_{n}^{-1/2}}-\mathbf{I})^{r}\mathbf{P}_{n}\mathbf{f}&=\mathbf{U}_{n}\textrm{diag}\Big((e^{i\sqrt{\frac{\lambda_{1}}{\lambda_{n}}}}-1)^r,\dots,(e^{i\sqrt{\frac{\lambda_{n}}{\lambda_{n}}}}-1)^r\Big)\mathbf{U}_{n}^{*}\mathbf{f}\\
 &=\sum_{j=1}^{n}(e^{i\sqrt{\frac{\lambda_{j}}{\lambda_{n}}}}-1)^r\hat{f}(j)\mathbf{u}_{j}.
 \end{align*}
 Then, 
  \begin{align*}
    \|(\mathbf{T}_{\lambda_{n}^{-1/2}}-\mathbf{I})^{r}\mathbf{P}_{n}f\|_{2}^2
    &=\sum_{j=1}^{n}|e^{i\frac{\sqrt{\lambda_{j}}}{\sqrt{\lambda_{n}}}}-1|^{2r}|\hat{f}(j)|^2\\
    &=\sum_{j=1}^{n}\Big(4\sin^2(\frac{\sqrt{\lambda_{j}}}{2\sqrt{\lambda_{n}}})\Big)^{r}|\hat{f}(j)|^2.
  \end{align*}
  Since $\frac 2 \pi\leq \frac{\sin x} x\leq 1$ for $0\leq x\leq 1/2$, we have  that for $j=1,2,\dots,n$,
  \begin{equation}\label{equ:sin_equivalence}
   \frac{4\lambda_{j}}{\pi^2\lambda_{n}}\leq4\sin^2(\frac{\sqrt{\lambda_{j}}}{2\sqrt{\lambda_{n}}})
    \leq \frac{\lambda_{j}}{\lambda_{n}}.
  \end{equation}
  Combining the fact that   $$\|\mathbf{L}^{r/2}\mathbf{P}_{n}f\|_{2}^{2}=\sum_{j=1}^{n}\lambda_{j}^r|\hat{f}(j)|^2,$$
  we have that
  $$(2/\pi)^{r}\lambda_{n}^{-r/2}\|\mathbf{L}^{r/2}\mathbf{P}_{n}f\|_{2}\leq
   \|(\mathbf{T}_{\lambda_{n}^{-1/2}}-\mathbf{I})^{r}\mathbf{P}_{n}f\|_{2}\leq
  \lambda_{n}^{-r/2}\|\mathbf{L}^{r/2}\mathbf{P}_{n}f\|_{2}.$$

   Finally, to see \eqref{equ:omega_vs_L}, one can just notice from the definition of $\omega_{r}(f,t)$ that 
   \[\|(\mathbf{T}_{s}-\mathbf{I})^{r}\mathbf{P}_{n}f\|_{2}\leq\|(\mathbf{T}_{s}-\mathbf{I})^{r}f\|_{2} \leq \omega_{r}(f,\lambda_{n}^{-1/2})\]
 for any $f\in\mathbb{C}^{N}$ and $|s|\leq \lambda_{n}^{-1/2}$. This proves the lemma.

\end{proof}



\begin{lemma}\label{lem:bound_for_one_freq}
  Let $n=2,3,\dots,N$. For any $f\in\mathbb{C}^{N}$ , we have
  $$\|(\mathbf{P}_{n}-\mathbf{P}_{n-1})f\|_{2}\leq C_{r}\omega_{r}(f,\lambda_{n}^{-1/2}),$$
which further implies that
  \begin{equation}\label{e-l}
  E_{n}(f)\leq C_{r}\sum_{k=n+1}^{N}\omega_{r}(f,\lambda_{k}^{-1/2}),
  \end{equation}
  where $
    C_{r}:=\Big(2\sin(\frac{1}{2})\Big)^{-r}$.
\end{lemma}
\begin{proof}
As in the proof of Lemma \ref{lem:1}, we see from $\mathbf{P}_{n}f-\mathbf{P}_{n-1}f=\hat{f}(n)\mathbf{u}_{n}$ that
   \begin{align*}
     \|(\mathbf{T}_{\lambda_{n}^{-1/2}}-\mathbf{I})^{r}\mathbf{P}_{n}f\|_{2}^2
    &=\sum_{j=1}^{n}|e^{i\frac{\sqrt{\lambda_{j}}}{\sqrt{\lambda_{n}}}}-1|^{2r}|\hat{f}(j)|^2\\
    &=\sum_{j=1}^{n}\Big(4\sin^2(\frac{\sqrt{\lambda_{j}}}{2\sqrt{\lambda_{n}}})\Big)^{r}|\hat{f}(j)|^2
    \geq \Big(4\sin^2(\frac{\sqrt{\lambda_{n}}}{2\sqrt{\lambda_{n}}})\Big)^{r}|\hat{f}(n)|^2 \\
     &=\Big(4\sin^2(\frac{1}{2})\Big)^{r}\|(\mathbf{P}_{n}-\mathbf{P}_{n-1})f\|_{2}^{2}.
  \end{align*}
 It implies that $$\|(\mathbf{P}_{n}-\mathbf{P}_{n-1})f\|_{2}\leq\Big(2\sin(\frac{1}{2})\Big)^{-r}\|(\mathbf{T}_{\lambda_{n}^{-1/2}}-\mathbf{I})^{r}\mathbf{P}_{n}f\|_{2}.$$
  Then, according to the definition of $\omega_{r}(f,t)$,
  \begin{align*}
    \omega_{r}(f,\lambda_{n}^{-1/2})=\sup_{|h|\leq\lambda_{n}^{-1/2}}\|(\mathbf{T}_{h}-\mathbf{I})^{r}f\|_{2}
    \geq \|(\mathbf{T}_{\lambda_{n}^{-1/2}}-\mathbf{I})^{r}f\|_{2}.
  \end{align*}
  Thus, $$\|(\mathbf{P}_{n}-\mathbf{P}_{n-1})f\|_{2}\leq\Big(2\sin(\frac{1}{2})\Big)^{-r}\omega_{r}(f,\lambda_{n}^{-1/2}),$$


 Next, bounding the best approximation error $$E_{n}(f)=\min\limits_{g\in PW_{\lambda_{n}}(\mathcal{G})}\|f-g\|_{2}=\|f-\mathbf{P}_{n}f\|_{2},$$ we have
  \begin{align*}
E_{n}(f)                                          &\leq\sum_{k=n+1}^{N}\|(\mathbf{P}_{k}-\mathbf{P}_{k-1})f\|_{2}\\
&\leq C_{r}\sum_{k=n+1}^{N}\omega_{r}(f,\lambda_{k}^{-1/2}).
\end{align*}
This proves Lemma \ref{lem:bound_for_one_freq}.
\end{proof}


{\bf Remark 1:} Lemma~\ref{lem:bound_for_one_freq} introduces a new upper bound for $E_{n}(f)$ in terms of $\omega_{r}(f,\lambda_{n}^{-1/2})$ as expressed in \eqref{e-l}. In the particular case of $n=N-1$, Lemma~\ref{lem:bound_for_one_freq} gives
$$E_{N-1}(f)\leq C_{r}\omega_{r}(f,\lambda_{N}^{-1/2})\leq C_{r}\omega_{r}(f,\lambda_{N-1}^{-1/2}).$$
When $r$ increases sufficiently, as $C_r \leq C_r'$, this result demonstrates an improvement over the estimation provided by Lemma~\ref{lem:best_app}

Now, we are in the position to show the equivalent relation between $K_{r}(f,t)$ and $\omega_{r}(f,t)$.

\begin{theorem}\label{thm:main}
For any $f\in\mathbb{C}^{N}$ and $t>0$, there exists a constant $C_1$ such that
$$C_{1}K_{r}(f,t)\leq\omega_{r}(f,t)\leq2^{r}K_{r}(f,t),~~\forall~t>0.$$
  \end{theorem}

\begin{proof}
 We first prove the upper bound of $\omega_{r}(f,t)$ in terms of $K$-functional. According to the properties of $\omega_{r}(f,t)$, for any $g\in\mathbb{C}^{N}$,
        \begin{align*}
              \omega_{r}(f,t)&\leq\omega_{r}(f-g,t)+\omega_{r}(g,t)\\
                                 &\leq2^{r}\|f-g\|_{2}+t^r\omega_{0}(\mathbf{L}^{r/2}g,t)\\
                                 &\leq2^{r}\Big(\|f-g\|_{2}+(\frac{t}{2})^{r}\|\mathbf{L}^{r/2}g\|_{2}\Big).
        \end{align*}
        Taking infimum over $g\in\mathbb{C}^{N}$ yields $\omega_{r}(f,t)\leq2^{r}K_{r}(f,t)$.
        
    Next, we prove the lower bound of $\omega_{r}(f,t)$ in terms of $K$-functional. The proof is provided in the following three cases: 
    \begin{itemize}
      \item [(i)] if $0< t\leq\lambda_{N}^{-1/2}$,  we have from $\mathbf{P}_{N}f=f$ that 
      \begin{align*}
      K_{r}(f,t)&\leq\|f-\mathbf{P}_{N}f\|_{2}+(\frac{t}{2})^{r}\|\mathbf{L}^{r/2}\mathbf{P}_{N}f\|_{2} \\
                & =(\frac{t}{2})^{r}\|\mathbf{L}^{r/2}\mathbf{P}_{N}f\|_{2}
                 \leq (\frac{\pi t\lambda_{N}^{1/2}}{4})^{r}\omega_{r}(f,\lambda_{N}^{-1/2}),
    \end{align*}
    where the last equality follows \eqref{equ:omega_vs_L}. Then, according to the first property of $\omega_{r}(f,t)$, we have
    \begin{align*}
      K_{r}(f,t)&\leq (\frac{\pi t\lambda_{N}^{1/2}}{4})^{r}\omega_{r}(f,\lambda_{N}^{-1/2})\\
                & \leq (\frac{\pi t\lambda_{N}^{1/2}}{4})^{r}(\frac{1}{t\lambda_{N}^{1/2}}+1)^{r}\omega_{r}(f,t)\\
                & \leq (\frac{\pi}{4}+\frac{\pi t\lambda_{N}^{1/2}}{4})^{r}\omega_{r}(f,t)\leq (\frac{\pi}{2})^{r}\omega_{r}(f,t).
    \end{align*}
      \item [(ii)] if $\lambda_{N}^{-1/2}\leq t\leq\lambda_{2}^{-1/2}$,
      then there exist some $n\in\{3,4,\dots,N\}$ such that $\lambda_{n}^{-1/2}\leq t\leq\lambda_{n-1}^{-1/2}$. It follows that $t\lambda^{1/2}_{n}\leq M$, where $$M:=\max\limits_{n=3,4,\dots,N}(\frac{\lambda_{n}}{\lambda_{n-1}})^{1/2}.$$ Hence
    \begin{align*}
      K_{r}(f,t)&\leq\|f-\mathbf{P}_{n}f\|_{2}+(\frac{t}{2})^{r}\|\mathbf{L}^{r/2}\mathbf{P}_{n}f\|_{2} \\
                & = \textrm{E}_{n}(f)+(\frac{t}{2})^{r}\|\mathbf{L}^{r/2}\mathbf{P}_{n}f\|_{2}.
    \end{align*}
    Then, according to Lemma \ref{lem:best_app} and \eqref{equ:omega_vs_L},
    \begin{align*}
    K_{r}(f,t)     & \leq C'_{r}\omega_{r}(f,\lambda_{n}^{-1/2})
                +(\frac{\pi t\lambda_{n}^{1/2}}{4})^{r}\omega_{r}(f,\lambda_{n}^{-1/2})\\
                & \leq \Big(C'_{r}+(\frac{\pi M}{4})^{r}\Big)\omega_{r}(f,\lambda_{n}^{-1/2})\leq \Big(C'_{r}+(\frac{\pi M}{4})^{r}\Big)\omega_{r}(f,t).
    \end{align*}
    \item [(iii)] if $t\geq\lambda_{2}^{-1/2}$, from the fact that $\mathbf{L}^{r/2}\mathbf{P}_{n}f=\mathbf{L}^{r/2}(\hat{f}(1)\mathbf{u}_{1})=\hat{f}(1)0\mathbf{u}_{1}=\mathbf{0}$ we have
    \begin{align*}
      K_{r}(f,t) &\leq \|f-\mathbf{P}_{1}f\|_{2}+(\frac{t}{2})^{r}\|\mathbf{L}^{r/2}\mathbf{P}_{1}f\|_{2}= \|f-\mathbf{P}_{1}f\|_{2}\\
                 & \leq \|f-\mathbf{P}_{2}f\|_{2}+\|(\mathbf{P}_{1}-\mathbf{P}_{2})f\|_{2} 
                  = E_{2}(f)+\|(\mathbf{P}_{1}-\mathbf{P}_{2})f\|_{2} \\
                  &\leq (C'_{r}+C_{r})\omega_{r}(f,\lambda_{2}^{-1/2})\leq (C'_{r}+C_{r})\omega_{r}(f,t),
    \end{align*}
     where the third inequality follows Lemma \ref{lem:best_app} and Lemma~\ref{lem:bound_for_one_freq}.
    \end{itemize}
    The proof of Theorem \ref{thm:main} is complete.
\end{proof}


In order to apply the deduced approximation results to GCNs, we extend the above results to graph signals with multiple channels and general Laplacian $\tilde{\b L}$.
Here $\tilde{\mathbf L}\in \RR^{N\times N}$ is symmetric positive semi-definite such as the normalized Laplacian matrix  $\mathbf{I} - \mathbf{D}^{-1/2} \mathbf{A} \mathbf{D}^{-1/2}$. Then the graph translation operator can be defined in the same way with \eqref{equ:translation} by replacing $\b L$ as $\tilde{\b L}$. For the graph signal $\mathbf{F}=(f_{1},f_{2},\dots,f_{m})\in\mathbb{C}^{N\times m}$ 
with $m$ channels, then we define its moduli by
$$\mathcal{W}_{r}(\mathbf{F},t; \tilde{\b L}):=\sum_{j=1}^{m}\omega_{r}(f_{j},t),~~t>0,$$
and correspondingly $K$-functional by
$$\mathcal{K}_{r}(\mathbf{F},t; \tilde{\b L}):=\inf_{g_{1},g_{2},\dots,g_{m}\in\mathbb{C}^{N}}\sum_{j=1}^{m}\|f_{j}-g_{j}\|_{2}+(\frac{t}{2})^{r}\sum_{j=1}^{m}\|\tilde{\mathbf{L}}^{r/2}g_{j}\|_{2},~~t>0.$$

With the similar argument as Theorem~\ref{thm:main}, we have
\begin{theorem}\label{thm:main_matrix}
  Let $\tilde{\b  L}$ be a symmetric positive semi-definite operator and $r\in\mathbb{N}$. With respect to $\tilde{\b L}$, for any $\mathbf{F}\in\mathbb{C}^{N\times m}$ and $t>0$, there exists a constant $C_{1}$ such that
 $$C_{1}\mathcal{K}_{r}(\mathbf{F},t;\tilde{\b L})\leq\mathcal{W}_{r}(\mathbf{F},t;\tilde{\b L})\leq2^{r}\mathcal{K}_{r}(\mathbf{F},t;\tilde{\b L}).$$
\end{theorem}

\section{A lower bound of approximation capability of GCNs} \label{sec:lower_bound}

In this section, we investigate the approximation capability of Graph Convolutional Networks (GCNs) by using the estimation presented in Theorem~\ref{thm:main}. Furthermore, we will provide a theoretical insight about over-smoothing phenomenon.

The GCNs under our consideration are structured with a graph filter $\mathbf{H}\in\mathbb{R}^{N\times N}$ and channels $\{m_{k}\}_{k=0}^{K}$
\begin{equation}\label{equ:GCN}
\mathbf{F}^{(k)}=\sigma\big(\mathbf{H}\mathbf{F}^{(k-1)}\mathbf{W}^{(k-1)}\big),
\end{equation}
for $k=1,2,\dots,K$, where $\sigma(x)=\max\{x,0\}$ is the ReLU activation function, $\mathbf{F}^{(k)}\in\mathbb{R}^{N\times m_{k}}$ is the output of the $k$-th layer (with the input feature $\mathbf{F}^{(0)}\in\mathbb{R}^{N\times m_{0}}$) and $\mathbf{W}^{(k-1)}\in\mathbb{R}^{m_{k-1}\times m_{k}}$ is the learnable matrix of weights of the $k$-th layer. 

This architecture is widely adopted in GCNs. Particularly, in many cases, the filter $\mathbf{H}$ is chosen to be a polynomial of the graph Laplacian $\mathbf{L}$ or the normalized Laplacian, as mentioned in \cite{2019Stable-GCN} and \cite{Li2018}. The filter, $\mathbf{H}$, possesses a low-frequency eigenvector that comprises nonnegative elements and does not oscillate around zero. With such filters, GCNs typically achieve good performances at limited depths but tend to suffer from the over-smoothing problem. In the following theorem, it is demonstrated that this occurs when low-frequency components are present and high-frequency eigenvalues are less than one. Recall the Frobenius norm of a matrix $\mathbf{A}=(A_{ij})\in\mathbb{R}^{m\times n}$ is given by $\|\mathbf{A}\|_{F}=(\sum_{i=1}^{m}\sum_{j=1}^{n}A_{ij}^{2})^{1/2}$.

 \begin{theorem}\label{thm:1ddecay}
 Suppose $\b F^{(0)}\in \RR^{N \times m_0}$ is a graph input signal with $N$ nodes and $m_0$ channels, and $\{\b F^{(k)}\}$ is  generated by a GCN \eqref{equ:GCN} with a symmetric filter $\mathbf{H}$ and weights $\b W^{(k-1)}\in \RR^{m_{k-1}\times m_{k}}$ for $k=1,2,\ldots,K$. Let $\{ (\mathbf{h}_i,\mu_i) \}_{i=1}^N $ form an orthonormal eigenvector-eigenvalue pairs of $\mathbf H$. 
 If $\mathbf{H}$ has a low frequency eigenvector with nonnegative components, denoted by $\mathbf{h}_1$, and $\|\b W^{(k-1)}\|_F\leq 1$ for all $k=1,\ldots, K$, then the high frequency parts of $\b F^{(K)}$ can be bounded by 
    \begin{align}\label{eq:lower-k+1}
    \begin{aligned}
        \sum_{j=1}^{m_K}\sum_{i=2}^N \left | \langle f_j^{(K)}, \b h_i\rangle \right |^2 
        \leq |\mu_{high}|^{2K} \|\b F^{(0)}\|_F^2,
    \end{aligned}
    \end{align}
    where  $\b F^{(K)} = (f_1^{(K)}, \dots, f_{m_K}^{(K)})\in \RR^{M\times m_K}$ and $|\mu_{high}|=\max\{|\mu_2|,\ldots,|\mu_N| \}$.
    
    Furthermore, if $|\mu_{high}|<1$, then for any $\b v \perp \b h_1$, $$\sum_{j=1}^{m_K}\left| \langle f_j^{(K)}, \b v \rangle \right | \rightarrow 0 , \quad K \rightarrow \infty. $$
\end{theorem}

\begin{proof}
For any $f\in\mathbb{R}^{N}$, we define
        $f_+:=\sigma(f)$ with $\sigma$ acting componentwise, $f_{-}:=f-f_{+}$
      and $\mathbf{P}f:= f-\langle f, \b h_1 \rangle \b h_1$ as the orthogonal projection onto span$\{\mathbf{h}_{j}\}_{j=2}^{N}$. We claim that
      \begin{equation*}\label{eq:lower_p}
    \|\mathbf{P}f\|_{2}^{2}\geq\|\mathbf{P}f_{+}\|_{2}^{2}
          =\|\mathbf{P}\sigma(f)\|_{2}^{2}, \quad \forall f\in\mathbb{R}^{N}.
      \end{equation*}
      We assume $\|\mathbf{P}f_{+}\|_{2}\neq 0$, otherwise \eqref{eq:lower_p} holds true trivially.
       Then, take an orthonormal basis $\{\mathbf{v}_1, \mathbf{v}_2 ,\dots, \mathbf{v}_{N-1}\}$ of the orthogonal complement of span$\{\mathbf{h}_{1}\}$, where
   \begin{align*}
        \mathbf{v}_1 := \frac{\P f_{+}}{\left \|\P f_{+} \right \|_2} = \frac{f_+ - \langle f_+, \b h_1 \rangle \b h_1}{\left \|\P f_{+} \right \|_2}.
    \end{align*}
    Observe from the special form $\sigma(t)=\max\{t,0\}$ of ReLU that for each $j\in\{1,2,\dots,N\}$, either $(f_{+})_{j}=\sigma(f_{j})$ or $(f_{-})_{j}=f_{j}-\sigma(f_{j})=\min\{f_{j},0\}$ vanishes. Hence $\langle f_{+},f_{-}\rangle=0$. Thus, it is easy to verify that
    \begin{align*}
        \langle f_+, \mathbf{v}_1 \rangle = \|\P f_+\|_2\geq 0~\text{and}~
        \langle f_-, \mathbf{v}_1 \rangle = -\frac{\langle f_+, \b h_1 \rangle \langle f_{-}, \b h_1 \rangle}{\left \|\P f_{+} \right \|_2} \geq 0,
    \end{align*}
    where we have used the assumption $\mathbf{h}_{1}\in\mathbb{R}_{+}^{N}$ and the observation $f_{-}\in\mathbb{R}_{-}^{N}$ which lead to $\langle f_{+},\mathbf{h}_{1}\rangle\geq0$ and $\langle f_{-},\mathbf{h}_{1}\rangle\leq0$.
    At this point, since $f_{+} = \sigma(f)$, we have
    \begin{align*}
    \nonumber    \|\P f\|^2_2 
        &\geq \left | \langle f, \mathbf{v}_1 \rangle \right|^2
        = \left|\langle f_{+}, \mathbf{v}_1 \rangle + \langle f_{-}, \mathbf{v}_1 \rangle\right|^2 \\
        &\geq|\langle f_{+},\mathbf{v}_{1}\rangle|^{2}= \|\P f_{+}\|^2_2 = \|\P\sigma(f)\|^2_2,
    \end{align*}
    where the first inequality is obtained by the Plancherel Theorem. This verifies our claim.
      
    Now, we are ready to show \eqref{eq:lower-k+1}. 
   Let $k\in\{1,2,\dots,K\}$. Since $\P$ is the orthogonal projection onto span$\{\mathbf{h}_{j}\}_{j=2}^{N}$, it is easy to see that
 \begin{align*}
    \begin{aligned}
        &\sum_{j=1}^{m_{k}}\sum_{i=2}^N \left | \langle f_j^{(k)}, \b h_i\rangle \right |^2 \\
        =&\sum_{j=1}^{m_{k}}\|\mathbf{P}f_{j}^{(k)}\|_{2}
        ^{2}= \sum_{j=1}^{m_k} \|\P\sigma (\b H\b F^{(k-1)} \b W^{(k-1)})_{j}\|_2^2 \\
        \leq& \sum_{j=1}^{m_{k}}\|\P (\b H\b F^{(k-1)} \b W^{(k-1)})_{j}\|_2^2= \sum_{j=1}^{m_k} \|\P \b H\b F^{(k-1)}  (\b W^{(k-1)})_{j}\|_2^2.
    \end{aligned}
    \end{align*}
     Viewing $\b H\b F^{(k-1)} (\b W^{(k-1)})_j$ as a vector in $\RR^N$, we see that the norm of its projection $\P \left(\b H\b F^{(k-1)} (\b W^{(k-1)})_j \right)$ can be expressed as $\|\P \b H\b F^{(k-1)} (\b W^{(k-1)})_j\|_2^2 = \sum_{i=2}^N |\langle \b h_i, \b H\b F^{(k-1)} (\b W^{(k-1)})_j \rangle|^2 =  \sum_{i=2}^N | \b h_i^\top  \b H\b F^{(k-1)} (\b W^{(k-1)})_j |^2  $. Hence, by the Cauchy-Schwarz inequality, 
    \begin{align*}
        &\sum_{j=1}^{m_k} \|\P \b H\b F^{(k-1)} (\b W^{(k-1)})_j\|_2^2 = \sum_{i=2}^N\sum_{j=1}^{m_k} |\b h_i^\top  \b H\b F^{(k-1)} (\b W^{(k-1)})_j |_2^2 \\
        \leq &\sum_{i=2}^N\sum_{j=1}^{m_k} \|\b h_i^\top  \b H\b F^{(k-1)} \|^2_2 \| (\b W^{(k-1)})_j\|_2^2 
        = \sum_{i=2}^N \|(\b H \b F^{k-1})^\top \b h_i \|^2_2 \| \b W^{(k-1)}\|_F^2.
    \end{align*}
     
     By our assumption, $\|\b W^{(k)}\|_F \leq 1$. Also, $\|(\b H \b F^{k-1})^\top \b h_i\|_2^2 = \| ( \b F^{k-1})^\top \b H^\top \b h_i  \|_2^2 = \| ( \b F^{k-1})^\top \mu_i \b h_i  \|_2^2 \leq |\mu_{high}|^2 \| ( \b F^{k-1})^\top \b h_i  \|_2^2 $.
    
Thus,
    \begin{align*}
        &\sum_{j=1}^{m_{k}}\sum_{i=2}^N \left | \langle f_j^{(k)}, \b h_i\rangle \right |^2 \leq  |\mu_{high}|^2 \sum_{i=2}^N \| ( \b F^{k-1})^\top \b h_i  \|_2^2 \\
        =& |\mu_{high}|^2 \sum_{j=1}^{m_{k-1}}\sum_{i=2}^N \left | \langle  f_j^{(k-1)}, \b h_i\rangle \right |^2 .
    \end{align*}


    Using this iteration relation verifies \eqref{eq:lower-k+1}.

    The second statement is a trivial consequence of \eqref{eq:lower-k+1}. This proves Theorem~\ref{thm:1ddecay}.
    
\end{proof}

We now apply Theorem~\ref{thm:1ddecay} to a study of approximation abilities of GCNs.
We use the following high-pass filter
$$\tilde{\b H}=[\b h_1,\tilde{\b h}_2,\ldots, \tilde{\b h}_N]\left(
  \begin{array}{cccc}
    0 & 0 & \cdots & 0 \\
    0 & \tilde{\mu}_2 & \ddots & 0 \\
    \vdots & \ddots & \ddots & \vdots \\
    0 & \cdots & \ddots & \tilde{\mu}_N \\
  \end{array}\right)[\b h_1, \tilde{\b h}_2\ldots, \tilde{\b h}_N]^T$$ 
to extract the high frequency part of a signal $\b F$, which has an orthonoraml basis $\{\b h_1, \tilde{\b h}_i, i=2,\dots,N\}$, with low-frequncy eigenvector the same as that of $\b H$ and high frequency eigenvalues $|\tilde{\mu}_i| \leq \tilde{\mu}_{high}$ for any $i$. Theorem~\ref{thm:1ddecay} implies high frequency part is squashed to zero and the following theorem provides a lower bound of GCN \eqref{equ:GCN}, indicating that there always exists approximation error which depends on the high frequency part of the signal, characterized by $\tilde{\b H}$. 

{\bf Remark 2:}
The selection of $\tilde{\mathbf{H}}$ is specifically to have the same low frequency eigenvector as that of $\mathbf{H}$. The choice of its high-frequency components do not affect our results. From this perspective, if the low-frequency eigenvector of $\mathbf{H}$ is the constant vector $(1/\sqrt{N}, \ldots, 1/\sqrt{N})^T$, then $\tilde{\mathbf{H}}$ can simply be the Laplacian $\mathbf L$, and the following modulus is defined with respect to the Laplacian.

\begin{theorem}[Lower bounds for approximating target signals.]\label{thm:lowerbound}
  Let  $\b F$ be a graph signal with $m$ channels and $\{\mathbf{F}^{(k)}\}_{k=0}^{K}$ be defined by \eqref{equ:GCN} with filter $\b H$. Under the same assumption as in Theorem~\ref{thm:1ddecay} and $m_K=m$,
     the following lower bound holds for any $t>0$ and  $r\in \mathbb{Z}_{+}$,
    \begin{align*}\label{eq:low2}
         \|\b F-\b F^{(K)}\|_F
        \geq  2^{-r}m^{-1/2} \mathcal{W}_r(\b F,t;\tilde{\b H}) - 2^{-r } t^r|\mu_{high}|^{K} |\tilde{\mu}_{high}|^{r/2}  \|\b F^{(0)}\|_F.
    \end{align*}
       In particularly, if $|\mu_{high}|< 1$, 
    \begin{equation*}\label{eq:low2_infty}
        \lim\limits_{K\rightarrow \infty} \|\b F-\b F^{(K)}\|_F \geq 2^{-r} m^{-1/2}\mathcal{W}_r(\b F,t;\tilde{\b H}).
    \end{equation*}
\end{theorem}


\begin{proof}
       Noticing that for $\tilde{\b H}$, the eigenvalue associating with $\b h_1$ is $0$, we have by Theorem~\ref{thm:1ddecay} for $\tilde{\b H}^{r/2}\b F^{K}:= ( \tilde{\b H}^{r/2} f_1^{K}, \dots,  \tilde{\b H}^{r/2} f_m^{K} )$,
    \begin{align*}
        &\|\tilde{\mathbf{H}}^{r/2} \b F^{(K)}\|^2_F 
        = \sum_{j=1}^m \|\tilde{\mathbf{H}}^{r/2} f_j^{(K)}\|^2_2 
        = \sum_{j=1}^m \left \|\sum_{i=1}^N \tilde{\mu}_i^{r/2} \langle f_j^{(K)}, \tilde{\b h}_i\rangle \tilde{\b h}_i \right \|^2_2 \\
        =& \sum_{j=1}^m \sum_{i=2}^N \left | \tilde{\mu}_i^{r/2} \langle f_j^{(K)}, \tilde{\b h}_i\rangle \right |^2 
        \leq |\tilde{\mu}_{high}|^r \sum_{j=1}^m\sum_{i=2}^N \left | \langle f_j^{(K)}, \tilde{\b h}_i\rangle \right |^2 \\
       =& |\tilde{\mu}_{high}|^r \sum_{j=1}^m\sum_{i=2}^N \left | \langle f_j^{(K)}, \b h_i\rangle \right |^2 
        \leq |\tilde{\mu}_{high}|^{r} |{\mu}_{high}|^{2K} \| \b F^{(0)}\|^2_F,
    \end{align*}
    Now, we compare $\|\b F-\b F^{(K)}\|_F + \left(\frac{t}{2}\right)^r\|\tilde{\b H}^{r/2} \b F^{(K)} \|_F$ with $\|\b F-\b G\|_F + \left(\frac{t}{2}\right)^r\|\tilde{\b H}^{r/2}\b G\|_F$ when the matrix $\b G$ runs over $\RR^{N\times m}$, and bound the difference between $\b F - \b F^{(K)}$ from below, for any $t>0$,
    \begin{align*}
        &\|\b F-\b F^{(K)}\|_F \\
        \geq& \inf_{\b G\in \RR^{N\times m}} \left \{ \|\b F-\b G\|_F + \left(\frac{t}{2}\right)^r\|\tilde{\b H}^{r/2}\b G\|_F \right\} - \left(\frac{t}{2}\right)^r \|\tilde{\b H}^{r/2}\b F^{(K)}\|_F \\
        \geq & m^{-1/2}\mathcal{K}_{r}(\mathbf{F},t;\tilde{\b H}) - \left(\frac{t}{2}\right)^r |\tilde{\mu}_{high}|^{r/2} |\mu_{high}|^{K}\|\b F^{(0)}\|_F ,
    \end{align*}
    where we have used the bounds $\sum_{j=1}^m \|f_j - g_j\|_2 \leq \sqrt{m} \sqrt{\sum_{j=1}^m \|f_j - g_j\|_2^2} = \sqrt{m} \|\b F - \b G\|_F$ and $\sum_{j=1}^m \| \tilde{\b H}^{r/2} g_j \|_2 \leq \sqrt{m} \| \tilde{\b H}^{r/2} \b G \|_F$. It follows by applying Theorem~\ref{thm:main_matrix} with $\tilde{\b L} = \tilde{\b H}$ that 
    \begin{align*}
        \|\b F-\b F^{(K)}\|_F
        \geq  2^{-r} m^{-1/2}\mathcal{W}_r(\b F,t;\tilde{\b H}) - \left(\frac{t}{2}\right)^r |\tilde{\mu}_{high}|^{r/2} |\mu_{high}|^{K}\|\b F^{(0)}\|_F.
    \end{align*}
\end{proof}



{\bf Remark 3:} Considering the relationship between GCNs and regular CNNs \cite{Survey_ZhangCQ,NN2023}, GCNs can be conceptualized as a variant where each channel employs a uniform kernel. The role of channels in regular CNNs was studied by a time-frequency analysis method in \cite{Zhou2018}.

 Since  $\omega_{0}(f,t) = \|f\|_2$ 
 for any $f \in \RR^N$ and $t>0$, applying Theorem~\ref{thm:lowerbound} with $r=0$, we can have the following corollary.
 \begin{corollary}
    Taking the same assumption as in Theorem~\ref{thm:lowerbound}, we have that for any graph signal $\b F=(f_1,f_{2}\ldots,f_m)\in \RR^{N\times m}$,
    \[ \|\b F-\b F^{(K)}\|_F\geq m^{-1/2}\sum_{j=1}^m\|f_{j}\|_2 - |\mu_{high}|^{K} \|\b F^{(0)}\|_F. \]
\end{corollary}

{\bf Remark 4:} Several widely used filters align with our assumptions, including ${\bf H}_{gcn} := ({\bf D} + {\bf I})^{-1/2}({\bf A} + {\bf I})({\bf D} + {\bf I})^{-1/2}$ and ${\bf H}_{sym} := {\bf I} - \alpha {\bf D}^{-1/2}{\bf L}{\bf D}^{-1/2}$. These filters have eigenvalues within the range $(-1,1]$, making them prone to the over-smoothing issue. The validity of this conclusion is further corroborated by numerical experiments discussed in the subsequent section.

\section{Experiments} \label{sec:Exp}
In this section, we conduct experiments to verify our theoretical results. In the experiments, we use the stochastic block model (SBM) with two classes characterised by $\mathbf{y}\in\{\pm{1}\}^{N}$ to generate the graph structure, where edges $(i,j)$ are added independently with probability 
$p\in(0,1]$ if $y_{i}=y_{j}$, and with probability $q\in[0,p)$ if $y_{i}\neq y_{j}$. Then, we obtain a random binary adjacency matrix $\mathbf{A}\in\mathbb{R}^{N\times N}$. Furthermore, the node features $\mathbf{F}^{(0)}$ are sampled from a Gaussian mixture model (GMM), i.e., 
$$\mathbf{F}^{(0)}=\mathbf{y}\boldsymbol{\mu}^{\top}+\boldsymbol{\epsilon}\in\mathbb{R}^{N\times N},$$
where $\boldsymbol{\mu}\in\mathbb{R}^{N}$ and $\boldsymbol{\epsilon}=(\epsilon_{ij})\in\mathbb{R}^{N\times N}$ with $\epsilon_{ij}\stackrel{i.i.d}{\sim}\mathcal{N}(0,\sigma^{2})$ and $\sigma\in(0,+\infty)$. Unless otherwise stated, we set $p=0.8$, $q=0.3$, $\sigma=10$, and $\boldsymbol{\mu}$ is generated randomly.

For simplicity, the experiments are conducted on feed-forward GCNs to verify our theoretical findings (i.e., Theorem \ref{thm:1ddecay}) for over-smoothing problems. Therefore, in the experiments,  we randomly select the weight matrix $\b W^{(k)}$ of \eqref{equ:GCN} based on the normal distribution $\mathcal{N}(0, 1)$ and then normalize it to satisfy that $\|\b W^{(k)}\|_F = 10$. Then, for each layer $k$, we evaluate the high frequency energy of $\b F^{(k)}$ according to the filter $\b H$:
$$E_h(\b F^{(k)}):= \sum_{j=1}^{m_k}\sum_{i=2}^N \left | \langle f_j^{(k)}, \b h_i\rangle \right |^2,$$
where $\mathbf{h}_{i}$ is the eigenvector of filter $\mathbf{H}$ and we set $m_k=N$ for all $k$.




\begin{figure}[htbp] 
\centering 
\includegraphics[width=0.7\textwidth]
{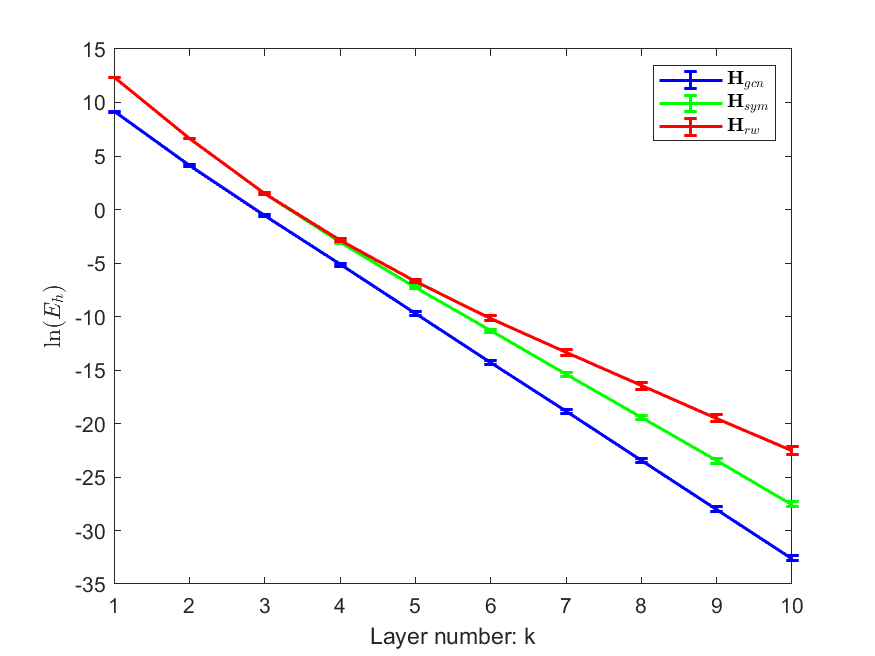}
\caption{Logarithmic high frequency energy of GCNs with different filters.} 
\label{fig:filters-log} 
\end{figure}



\subsection{Over-smoothing problems in classical GCNs}
First, we test the effect of three different filters on the over-smoothing problem in classical GCNs, namely: $\b H_{gcn}$, $\b H_{sym}$, and $\b H_{rw}=\b I - \alpha \b D^{-1}\b L$ with $\alpha = 0.75$. 
To this end, we conduct the experiments over 1000 independent trials on the graph of size $N=1000$. In the experiments, we compute $E_{h}(\mathbf{F}^{(k)})$ for each layer and then plot the results of $\ln(E_h)$ together with its error bar as a function of layer $k$, as shown in Figure \ref{fig:filters-log}. From Figure \ref{fig:filters-log}, we can see that the value of high frequency energy decays exponentially, which coincides with Theorem~\ref{thm:1ddecay}. This causes the fact that the values of high frequency energy almost equal zero as $k$ increases.

\begin{figure}[htbp] 
\centering 
\includegraphics[width=0.7\textwidth]
{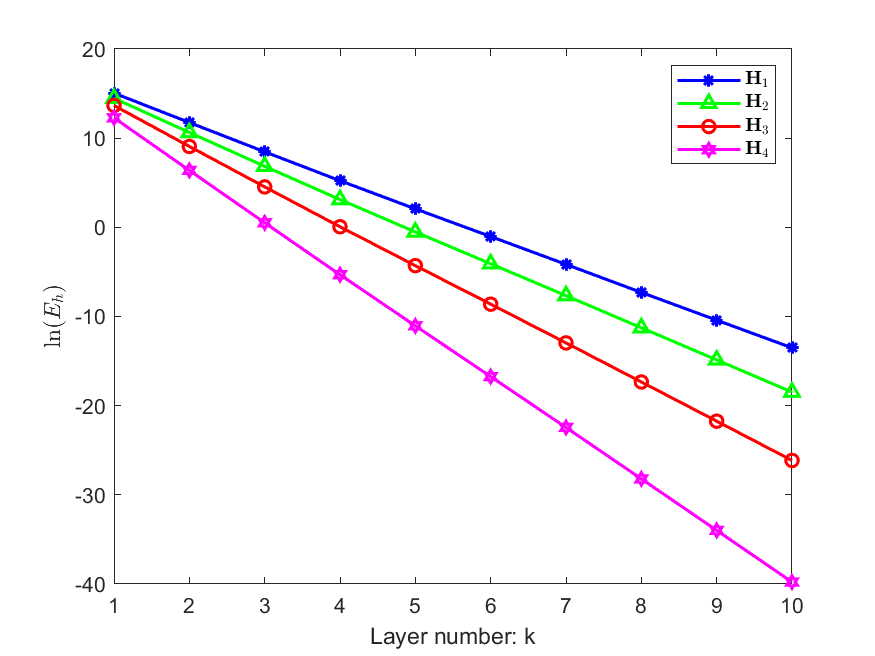}
\caption{Logarithmic high frequency energy of GCNs with $\b H_j, j=1,2,3,4$.} 
\label{fig:filters-Hi} 
\end{figure}

\subsection{Decay rate with respect to $|\mu_{high}|$}
According to Theorem \ref{thm:1ddecay}, we have that $$E_h(\mathbf{F}^{(K)}) \leq  |\mu_{high}|^{2K}\|\mathbf{F}^{(0)}\|_{F}^{2},$$ that is, the decay rate of the upper bound of $E_h$ is dependent on $\mu_{high}$. To show this, we next conduct the experiments on the following four filters $\b H_j$, $j=1,\dots,4$:
$$\b H_j=[\b h_1,\ldots, \b h_N]\left(
  \begin{array}{cccc}
    1 & 0 & \cdots & 0 \\
    0 & a_j & \ddots & 0 \\
    \vdots & \ddots & \ddots & \vdots \\
    0 & \cdots & \ddots & a_j \\
  \end{array}\right)[\b h_1,\ldots, \b h_N]^T,$$ 
where $a_j = 1-0.25\times (j-1)<1$ and $\b h_1,\dots, \b h_N$ are chosen to be eigenvectors from $\b H_{gcn}$. At this point, $\mu_{high}$ of each $\mathbf{H}_{j}$ is $a_{j}$. In the experiments, we compute $E_{h}$ for each layer $k$ and then plot $\ln(E_{h})$ as a function of $k$, which can be seen in Figure~\ref{fig:filters-Hi}. From Figure~\ref{fig:filters-Hi} we observe that for a fixed $k$, $\mathbf{H}_{j}$ with smaller $\mu_{high}$ has the smaller $E_{h}$. Furthermore, it can be seen that as $\mu_{high}<1$ goes smaller, the curve of $\mathbf{H}_{j}$ becomes more steep, i.e., $E_{h}$ of $\mathbf{H}_{j}$ with a smaller $\mu_{high}$ decays faster, which is consistent with our Theorem \ref{thm:1ddecay}.


\subsection{Exploration to alleviate over-smoothing problems with skip-connections? }
In the existing works, a large number of results have shown that the skip-connection architecture helps to alleviate the over-smoothing problem of the network. In this part, based on our Theorem \ref{thm:1ddecay}, 
we conduct experiment to evaluate the ability of the following three skip-connection GCNs to alleviate the over-smoothing problem, namely: ResGCN, APPNP, and GCNII, which are defined iteratively as follows: 
\begin{align*}
    &\mathbf{F}^{(k+1)}_{\text{ResGCN}} = \sigma \left( \mathbf{H}_{gcn}\b F^{(k) }_{\text{ResGCN}} \b W^{(k)} \right) + \b F^{(k)}_{\text{ResGCN}}, \\
    &\mathbf{F}^{(k+1)}_{\text{APPNP}} = (1-\alpha_k)  \mathbf{H}_{gcn}\b F^{(k) }_{\text{APPNP}}  + \alpha_k \b F^{(0) } \b W^{(k)}, \\
    &\mathbf{F}^{(k+1)}_{\text{GCNII}} = \sigma \left( \left((1-\alpha_k)\mathbf{H}_{gcn}\b F^{(k) }_{\text{GCNII}} + \alpha_k \b F^{(0) } \right) \left(\beta_k\b W^{(k)} + (1-\beta_k)\b I \right) \right),
\end{align*}
and for the final layer,
\begin{align*}
    &\mathbf{F}^{(K)}_{\text{ResGCN}} = \mathbf{H}_{gcn}\b F^{(K-1)}_{\text{ResGCN}} \b W^{(K-1)} + \b F^{(K-1)}_{\text{ResGCN}}, \\
    &\mathbf{F}^{(K)}_{\text{APPNP}} = (1-\alpha_{K-1})  \mathbf{H}_{gcn}\b F^{(K-1) }_{\text{APPNP}}  + \alpha_k \b F^{(0) } \b W^{(K-1)}, \\
    &\mathbf{F}^{(K)}_{\text{GCNII}} = \left((1-\alpha_{K-1})\mathbf{H}_{gcn}\b F^{(K-1) }_{\text{GCNII}} + \alpha_{K-1} \b F^{(0) } \right) \left(\beta_{K-1}\b W^{(K-1)} + (1-\beta_{K-1})\b I \right) ,
\end{align*}
 

The experiments are conducted on the graph of size $N=100$, and we set $\alpha_{k}=\beta_{k}=0.5$ for $k=0,1,\dots,K-1$. Furthermore, for the sake of comparison, we evaluate energies $E_i$ and $E_h$ regarding to the normalized $\b F^{(K)}_{\text{ResGCN}}$, $\b F^{(K)}_{\text{APPNP}}$ and $\b F^{(K)}_{\text{GCNII}}$ , where $E_i$ denotes the frequency domain information as defined in the following.
\begin{equation}\label{equ:Ei}
\nonumber E_i(\b F^{(K)}):=E_i(\b F^{(K)};\b H) := \sum_{j=1}^{m_K} \left | \langle f_j^{(K)}, \b h_i\rangle \right |^2,
\end{equation}
which characterizes the frequency energy along the direction $\b h_i$ over all $m_K = 100$ channels for output $\b F^{(K)}$ under filter $\b H$. 

Based on Theorem \ref{thm:1ddecay}, we first compute the high frequency energy $E_{h}$ for the three GCNs and then list the results of $\ln(E_{h}(\mathbf{F}^{(K)}))$ in Table \ref{tab:high_fre_three_nets}. From Table \ref{tab:high_fre_three_nets}, we observe that as the depth $K$ increases, the high frequency parts of both APPNP and GCNII are always dominant, while the high frequency parts of ResGCN are not. This shows that APPNP and GCNII perform better than ResGCN in alleviating the over-smoothing problem, which is consistent with the results of the existing works.


\begin{table}[!t]
\small
\begin{tabular}{|c|c|c|c|c|c|c|c|}
\hline
  $K$    & 1        & 5        & 10       & 20       & 30       & 40       & 50       \\ \hline
ResGCN & -0.0193 & -0.3469 & {\bf -2.9945} & {\bf -10.465} & {\bf -17.68} & {\bf -22.94} & {\bf -26.95} \\ \hline
APPNP  & -0.0209 & -0.0176 & -0.0119 & -0.0151 & -0.0163 & -0.0135 & -0.0163 \\ \hline
GCNII  & -0.0408 & -0.0452  & -0.04 & -0.0412 & -0.0473 & -0.0501 & -0.0408 \\ \hline
\end{tabular}
\caption{Logarithmic high frequency energy $\ln(E_h)$ of ResNet, APPNP and GCNII.} 
\label{tab:high_fre_three_nets}
\end{table}

 Furthermore, to analyze the energy distribution in the frequency domain, we plot the histogram of energies $E_{i}$ for the normalized $\mathbf{F}^{(K)}$ with $K=50$ in Figure \ref{fig:hist_three_nets}. From the results of Figure \ref{fig:hist_three_nets} we can see that, both APPNP and GCNII have energy distribution at each frequency, while ResGCN's energy distribution is mainly concentrated in the low frequency part and the energy in the high frequency part is almost negligible. This again explains the limited ability of ResGCN to alleviate the over-smoothing problem in deep networks.

Combining the experimental results in this part and the structure of the three networks, we can conclude that there may exist two effective avenues to alleviate the over-smoothing problem in GCNs, namely: 1) removing/replacing ReLU activation function, and 2) retaining some major high frequency of the input features.

\begin{figure}[!t] 
\centering 
\subfigure[ResGCN]{\includegraphics[width=0.32\textwidth]
{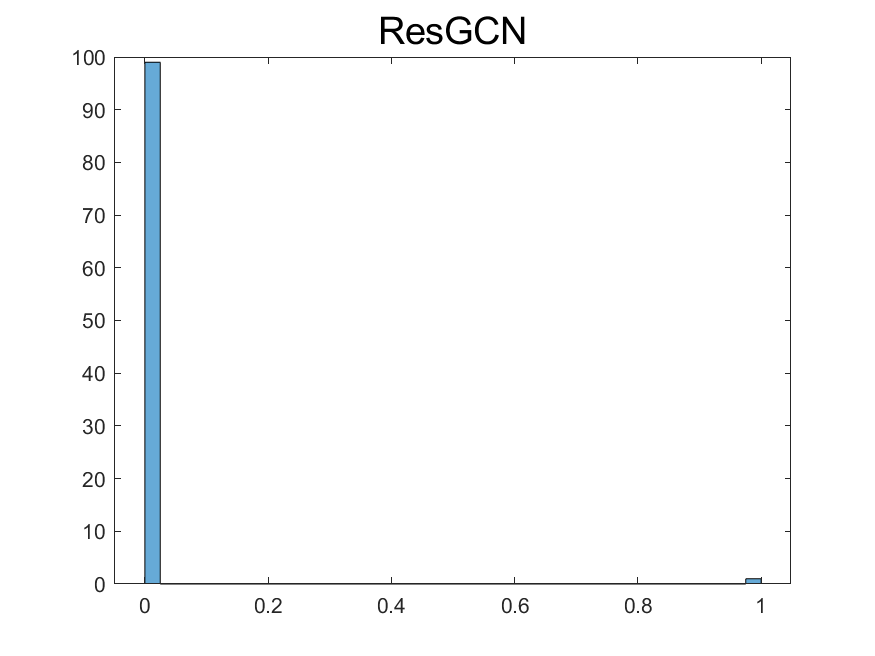}}
\subfigure[APPNP]{\includegraphics[width=0.32\textwidth]
{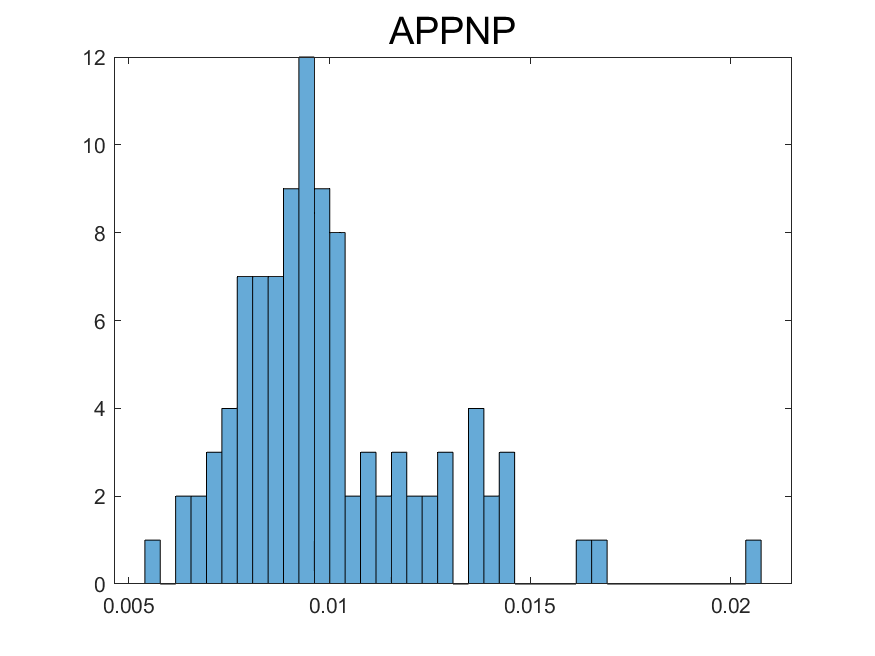}}
\subfigure[GCNII]{\includegraphics[width=0.32\textwidth]
{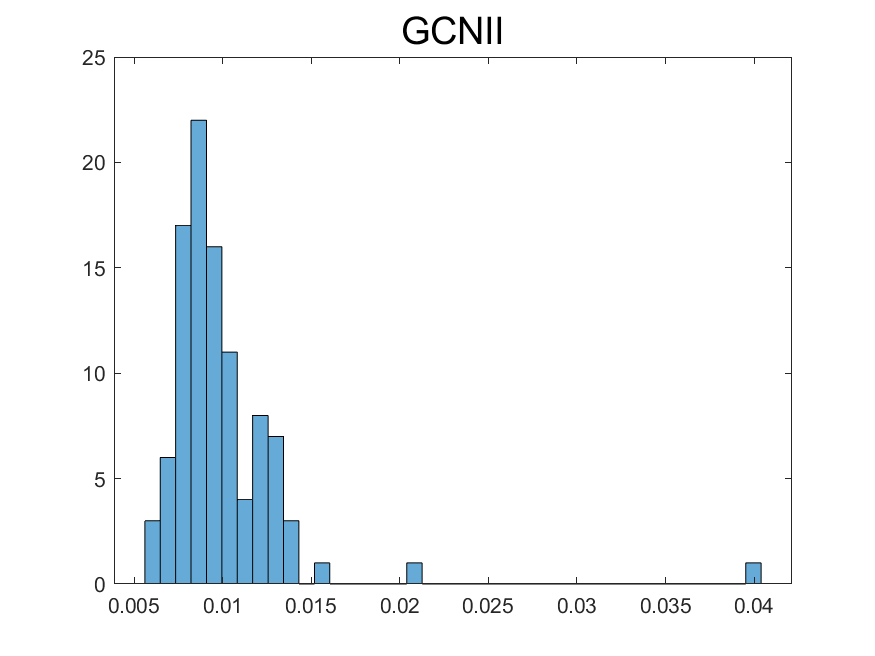}} 
\caption{Histogram of energies in the frequency domain, evaluating distribution of outputs of ResNet, APPNP and GCNII.} 
\label{fig:hist_three_nets} 
\end{figure}

\section{Conclusion and Discussion}\label{sec:conclusion}


This study has explored the approximation theory of functions on graphs, leveraging the $K$-functional to enhance our understanding of Graph Convolutional Networks (GCNs). Our findings contribute to the theoretical underpinnings of GCNs, particularly in terms of their abilities and limitations in approximating functions and managing the over-smoothing phenomenon.

The insights from this study highlight the need to maintain a balance between low and high-frequency information during the forward pass of data through GCNs. Such balance is crucial for mitigating over-smoothing and enhancing the approximation capabilities of these networks. Potential strategies derived from our findings to address these challenges include:
\begin{enumerate}
\item Incorporating Residual Connections: Adding residual connections between the input layer and intermediate layers to GCNs can help preserve rich information, allowing the network to maintain access to high-frequency details that might otherwise diminish through successive layers.
\item	Enhancing the Filter Channels: Introducing additional or adaptive filter channels can provide finer control over how frequency components are processed, enabling the network to selectively emphasize or de-emphasize certain frequencies based on the task requirements.
\item	Separating Different Frequencies for Processing: Implementing mechanisms to process different frequency components separately instead of only employing low-pass filters can prevent the loss of crucial information, particularly the high-frequency signals, thus reducing the risk of over-smoothing. 
\end{enumerate}

\bibliographystyle{elsarticle-num-names}
\biboptions{square,numbers,sort&compress}
\bibliography{mybibfile}
\end{document}